\documentclass[sigconf]{acmart}

\AtBeginDocument{%
  \providecommand\BibTeX{{%
    \normalfont B\kern-0.5em{\scshape i\kern-0.25em b}\kern-0.8em\TeX}}}

\usepackage{hyperref}       
\usepackage{url}            
\usepackage{booktabs}       
\usepackage{nicefrac}       
\usepackage{multirow}
\usepackage{enumerate}
\usepackage{caption}
\usepackage{subcaption}
\usepackage{graphicx}
\usepackage{xspace}

\usepackage{mathtools,commath,amsmath,amssymb,amsfonts,bm}\interdisplaylinepenalty=2500
\usepackage{amsthm}
\usepackage{algpseudocode,algorithm}
\usepackage{tabularx,adjustbox}
\newcolumntype{L}[1]{>{\hsize=#1\hsize\raggedright\arraybackslash}X}%
\newcolumntype{R}[1]{>{\hsize=#1\hsize\raggedleft\arraybackslash}X}%
\newcolumntype{C}[1]{>{\hsize=#1\hsize\centering\arraybackslash}X}%
\usepackage{todonotes}
\usepackage{breqn}
\clearpage
\extrafloats{100}
\RequirePackage{tcolorbox}
\tcbuselibrary{breakable}
\tcbset{parbox=false, left=1ex, right=1ex}

\usepackage{tikz,pgfplots}
\usepgfplotslibrary{groupplots}
\usetikzlibrary{positioning}
\usetikzlibrary{calc}
\usetikzlibrary{fit}
\usetikzlibrary{decorations.text}
\usetikzlibrary{decorations.pathreplacing}
\usetikzlibrary{shapes.misc}
\usetikzlibrary{shapes.geometric,arrows}
\usetikzlibrary{decorations.pathmorphing}

\newcommand{\ie}{\textit{i.e.},\xspace}

\newcommand{\mech}{\ensuremath{\mathcal{A}}\xspace}
\newtheorem{theorem}{Theorem}
\newtheorem{definition}{Definition}[section]

\algnewcommand\algorithmicinput{\textbf{Input:}}
\algnewcommand\Input{\item[\algorithmicinput]}
\algnewcommand\algorithmicoutput{\textbf{Output:}}
\algnewcommand\Output{\item[\algorithmicoutput]}
\algnewcommand\algorithmictier{\textbf{Role:}}
\algnewcommand\Tier{\item[\algorithmictier]}

\copyrightyear{2020} 
\acmYear{2020} 
\setcopyright{acmcopyright}\acmConference[SIGIR '20]{Proceedings of the 43rd International ACM SIGIR Conference on Research and Development in Information Retrieval}{July 25--30, 2020}{Virtual Event, China}
\acmBooktitle{Proceedings of the 43rd International ACM SIGIR Conference on Research and Development in Information Retrieval (SIGIR '20), July 25--30, 2020, Virtual Event, China}
\acmPrice{15.00}
\acmDOI{10.1145/3397271.3401260}
\acmISBN{978-1-4503-8016-4/20/07}

\begin{document}
\fancyhead{}
\title{Towards Differentially Private Text Representations}

\author{Lingjuan Lyu}
\affiliation{
  \institution{National University of Singapore}
}
\email{lyulj@comp.nus.edu.sg}

\author{Yitong Li}
\affiliation{
  \institution{The University of Melbourne}
}
\email{yitongl4@student.unimelb.edu.au}

\author{Xuanli He}
\affiliation{
  \institution{Monash University}
}
\email{xuanli.he1@monash.edu}

\author{Tong Xiao}
\affiliation{
  \institution{Northeastern University}
}
\email{xiaotong@mail.neu.edu.cn}

\begin{abstract}
Most deep learning frameworks require users to pool their local data or model updates to a trusted server to train or maintain a global model. 
The assumption of a trusted server who has access to user information is ill-suited in many applications. To tackle this problem, we develop a new deep learning framework under an untrusted server setting, which includes three modules: (1) embedding module, (2) randomization module, and (3) classifier module.
For the randomization module, we propose a novel local differentially private (LDP) protocol to reduce the impact of privacy parameter $\epsilon$ on accuracy, and provide enhanced flexibility in choosing randomization probabilities for LDP. 
Analysis and experiments show that our framework delivers comparable or even better performance than the non-private framework and existing LDP protocols, demonstrating the advantages of our LDP protocol.
\end{abstract}


%


\maketitle
\section{Introduction}
\label{sec:introduction}
The proliferation of deep learning (DL) has led to notable success in natural language processing (NLP), meanwhile, a series of privacy and efficiency challenges arise~\cite{hovy2015user,li2018towards,coavoux2018privacy}.
In NLP tasks, the input text often provides sufficient clues to portray the authors, such as their genders, ages, and other important attributes. Concretely, sentiment analysis tasks often impose privacy-related implications on the authors whose text is used to train models, and user attributes can be easily detectable from online review data, as evidenced by~\cite{hovy2015user}. Private information can take the form of key phrases explicitly contained in the text. However, it can also be implicit~\cite{preoctiuc2015analysis}. For example, the input representation after the embedding layer, or the intermediate hidden representation may still carry sensitive information which can be exploited for adversarial usages. It has been justified that an attacker can recover private variables with higher than chance accuracy, using only the hidden representation~\cite{coavoux2018privacy,li2018towards}. Such attack would occur in scenarios where end users send their learned representations to the cloud for grammar correction, translation, or text analysis tasks~\cite{li2018towards}.

To protect privacy, previous efforts resorted to a trusted aggregator to ensure centralized DP (CDP)~\cite{abadi2016deep}.
On the other hand, when participants are reluctant to directly share their crowd-sourced data with the server, federated learning becomes a promising learning paradigm that pushes model training to the edge~\cite{mcmahan2016communication}.
However, running complex deep neural networks (DNNs) with millions of parameters comes with resource limitations and user experience penalties. Moreover, most federated learning frameworks still assume a trusted aggregator who can have access to local model parameters or gradients~\cite{mcmahan2016communication}. The recent work pointed out the limitation of the trusted server and the associated privacy issues~\cite{lyu2020towards,lyu2020threats}.
Without an untrusted server,~\citet{shokri2015privacy} proposed to blur local model gradients by adding noise using differential privacy.
However, their privacy bounds are given per-parameter, the gigantic amount of model parameters prevents their technique from providing a meaningful privacy guarantee. The other cryptograph-based methods can be resource-hungry or overly complex for users~\cite{bonawitz2017practical}. More recently,~\citet{li2018towards} and~\citet{coavoux2018privacy} proposed to train deep models with adversarial learning. However, both works provide only empirical 
privacy, without any formal privacy guarantees.

To address the aforementioned problems, we are inspired to take a different approach by utilizing LDP. 
\textbf{Our contributions} include: 
\begin{itemize}
\item We are the first to train on differentially private crowd-sourced representations for NLP tasks. We propose a novel LDP protocol to preserve the privacy of the extracted representation from user inputs. It offers enhanced flexibility in choosing the randomization probabilities in LDP.
\item Experimental results on various NLP tasks show that our framework delivers comparable or even better performance than the non-private framework, and our LDP protocol demonstrates advantages over the existing LDP protocols.
\end{itemize}

\section{Preliminaries and Related work}
\label{sec:PRELIMINARIES}
\subsection{Local Differential Privacy (LDP)}
For the scenario where data are sourced from multiple individuals, while the server is untrusted, LDP~\cite{duchi2013local} is needed to enable data owners to perturb their private data before publication. LDP roots in randomized response~\cite{warner1965randomized}, and it has been deployed in many real-world applications such as Google's Chrome browser, Apple's iOS, and US Census Bureau. A formal definition of LDP is provided in Definition~\ref{def:ldp}.

\begin{definition}
\label{def:ldp}
A randomized algorithm \mech satisfies $\epsilon$-LDP if and only if for any two input tuples $v$ and $v'$, we have
\begin{eqnarray*}
\Pr\{\mech(v)=o\} &\leq& \exp(\epsilon) \cdot \Pr\{\mech(v')=o\} 
\end{eqnarray*}

for $\forall o \in Range(\mech)$, where $Range(\mech)$ denotes the set of all possible outputs of the algorithm \mech. 
\end{definition}

\begin{figure}[!htp]
\centering
        \includegraphics[width=7cm]{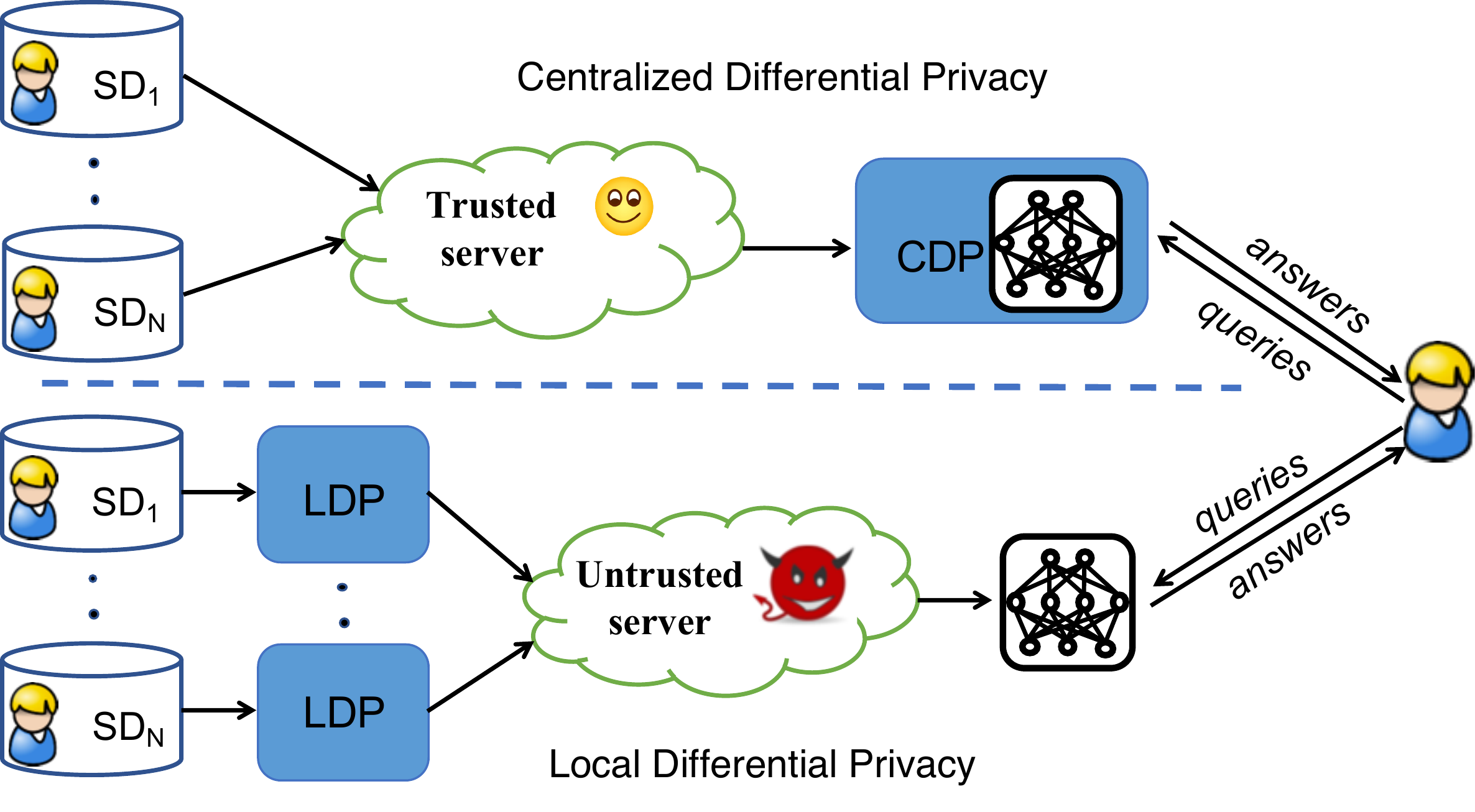}
        \caption{Deep Learning with CDP and LDP.}
\label{fig:LDPCDP}
\end{figure}

Compared to CDP~\cite{dwork2014algorithmic,abadi2016deep}, LDP offers a stronger level of protection. As illustrated in Figure~\ref{fig:LDPCDP}, in DL with CDP, the trusted server owns the data of all users~\cite{abadi2016deep}, and the server implements CDP algorithm before answering queries from end users. This approach can pose a privacy threat to data owners when the server is untrusted. 
Moreover, DL algorithms with CDP are inherently computationally complex~\cite{abadi2016deep}. By contrast, in DL with LDP, data owners are willing to contribute their data for social good, but do not fully trust the server, so it necessitates data perturbation before releasing it to the server for further learning.

\subsection{LDP Protocols}
\label{sec:ldp_protocols}
The most relevant LDP protocol is called Unary Encoding (UE), which consists of two steps~\cite{wang2017locally}: 

\textbf{Encoding}. Any single input $v$ is encoded into a $d$-bit vector ($d$ is domain size), where only the $v$-th bit equals to 1, \ie $\vec{B}$= Encode($v$), such that B[v]=1 and B[i]=0 for $i \neq v$. Hence each $d$-bit vector contains $d-1$ zeros and only 1 one, and the maximum difference between two adjacent binary vectors is 2, \ie sensitivity $\Delta f=2$. 

\textbf{Perturbing}. Each bit with value 1 is preserved with probability $p$, thus, Perturb($\vec{B}$) outputs $\vec{B}'$ as
\begin{equation}
    Pr[B'[i]=1]=
    \begin{cases}
      p, & \text{if}\ B[i]=1 \\
      q, & \text{if}\ B[i]=0
    \end{cases}
  \end{equation}
 
Here two key parameters in perturbation are $p=\Pr\{1 \rightarrow 1\}$, the probability that 1 remains 1 after perturbation, and $q=\Pr\{0 \rightarrow 1\}$, the probability that 0 is flipped to 1.

Depending on the choice of $p$ and $q$, UE based LDP protocols can be classified into~\cite{wang2017locally}:

\textbf{Symmetric Unary Encoding (SUE)}: 
SUE assumes the probability that a bit of 1 is preserved ($p$) equals the probability that a bit of 0 is preserved ($1-q$), \ie $p+q=1$, $p=\frac{e^{\epsilon/\Delta f}}{1+e^{\epsilon/\Delta f}}, q=\frac{1}{1+e^{\epsilon/\Delta f}}$.

\textbf{Optimized Unary Encoding (OUE)}: 
OUE optimizes SUE by using the optimized choices of $p, q$ for $\epsilon$-LDP. Setting $p$ and $q$ can be viewed as splitting 
$\epsilon$ into $\epsilon_1+\epsilon_2$ such that $\frac{p}{1-p}=e^{\epsilon_1}$ and $\frac{1-q}{q}=e^{\epsilon_2}$. That is, $\epsilon_1$and $\epsilon_2$ are the privacy budgets spent on transmitting 1's and 0's respectively. If there are more 0's than 1's in the encoded representation, it is reasonable to allocate as much privacy budget for transmitting the 0 bits as possible to maintain utility. In the extreme, setting $\epsilon_1=0$ and $\epsilon_2=\epsilon$ gives $p=\frac{1}{2}$ and $q=\frac{1}{1+e^{\epsilon/\Delta f}}$.

\section{Deep Learning with LDP}
\subsection{Optimized Multiple Encoding (OME)}
\label{sec:OME}
As both SUE and OUE are dependent on the domain size $d$, which may not scale well when $d$ is large. To remove the dependence on $d$, we propose a new LDP protocol called Optimized Multiple Encoding (OME). The key idea is to map each real value $v_i$ of the embedding vector into a binary vector with a fixed size $l$. Therefore, for the extracted embedding vector $\vec{v}=\{v_1,v_2,\cdots,v_r\}$ with $r$ elements, changing all elements of $\vec{v}$ results in $\Delta f=2r$ in both SUE and OUE, and $\Delta f=rl$ in OME. To enhance flexibility and utility in OME, we follow the intuition behind OUE~\cite{wang2017locally} to perturb 0 and 1 differently. 

In particular, we introduce a randomization factor $\lambda$ to adjust the randomization probabilities in OME. As implied in Theorem~\ref{theorem:MOME}, by increasing $\lambda$, we can decrease $q$, thus increasing the probability of keeping the original 0's. For the value of $p$, we increase the probability of preserving the original 1's for half of the bit vector while decreasing the corresponding probability for the other half. In this way, OME maintains both privacy and utility.

\begin{theorem}
\label{theorem:MOME}
For any inputs $v, v'$ and any encoded bit vector $B$ with sensitivity $rl$, OME provides $\epsilon$-LDP given 
\begin{align}\label{eq:MOME_p}
p = \Pr\{1 \rightarrow 1\} & = 
\begin{cases}
\frac{\lambda}{1+\lambda},\ $for$\ i \in 2n \\
\frac{1}{1+\lambda^3},\ $for$\ i \in 2n+1
\end{cases} \\
\label{eq:MOME_q}
q=\Pr\{0 \rightarrow 1\} & =\frac{1}{1+\lambda e^{\frac{\epsilon}{rl}}}
\end{align}
\end{theorem}

\begin{proof}
Let $v$ and $\vec{B}$ represent an input and its encoded bit representation. Given that $\vec{B}$ has a sensitivity of $rl$, the privacy budget $\epsilon$ needs to be divided by the sensitivity for each bit. By setting 

\begin{equation*}
p = \Pr\{1 \rightarrow 1\} = 
\begin{cases}
\frac{\lambda}{1+\lambda},\ $for$\ i \in 2n \\
\frac{1}{1+\lambda^3},\ $for$\ i \in 2n+1
\end{cases}
q=\Pr\{0 \rightarrow 1\}=\frac{1}{1+\lambda e^{\frac{\epsilon}{rl}}}
\end{equation*}

\begin{equation*}
1-p = \Pr\{1 \rightarrow 0\} = 
\begin{cases}
\frac{1}{1+\lambda},\ $for$\ i \in 2n \\
\frac{\lambda^3}{1+\lambda^3},\ $for$\ i \in 2n+1
\end{cases}
1-q=\Pr\{0 \rightarrow 0\}=\frac{\lambda e^{\frac{\epsilon}{rl}}}{1+\lambda e^{\frac{\epsilon}{rl}}}
\end{equation*}

Then for any inputs $v, v'$, we have
\begin{equation*}
\begin{small}
\begin{aligned}
&\frac{\Pr\{\vec{B}|v\}}{\Pr\{\vec{B}|v'\}}= \frac{\prod_{i=1}^{rl} \Pr\{B[i]|v\}}{\prod_{i=1}^{rl} \Pr\{B[i]|v'\}}=\frac{\prod_{i \in 2n} \Pr\{B[i]|v\}}{\prod_{i \in 2n} \Pr\{B[i]|v'\}} \times \frac{\prod_{i \in 2n+1} \Pr\{B[i]|v\}}{\prod_{i \in 2n+1} \Pr\{B[i]|v'\}}\\
\leq& \left(\frac{\Pr\{1 \rightarrow 1\}}{\Pr\{1 \rightarrow 0\}} \times \frac{\Pr\{0 \rightarrow 0\}}{\Pr\{0 \rightarrow 1\}}\right)_{i \in 2n}^{\frac{rl}{2}}  \times \left(\frac{\Pr\{1 \rightarrow 1\}}{\Pr\{1 \rightarrow 0\}} \times \frac{\Pr\{0 \rightarrow 0\}}{\Pr\{0 \rightarrow 1\}}\right)_{i \in 2n+1}^{\frac{rl}{2}}\\
=&\left(\frac{\frac{\lambda}{1+\lambda}}{\frac{1}{1+\lambda}} \times \frac{\frac{\lambda e^{\frac{\epsilon}{rl}}}{1+\lambda e^{\frac{\epsilon}{rl}}}}{\frac{1}{1+\lambda e^{\frac{\epsilon}{rl}}}}\right)^{\frac{rl}{2}} \times \left(\frac{\frac{1}{1+\lambda^3}}{\frac{\lambda^3}{1+\lambda^3}} \times \frac{\frac{\lambda e^{\frac{\epsilon}{rl}}}{1+\lambda e^{\frac{\epsilon}{rl}}}}{\frac{1}{1+\lambda e^{\frac{\epsilon}{rl}}}}\right)^{\frac{rl}{2}}=e^\epsilon
\end{aligned}
\end{small}
\end{equation*}
\end{proof}

\subsection{Framework Realization}
As shown in Figure~\ref{fig:LDPDL}, the general setting for our proposed deep learning with LDP consists of three main modules: (1) embedding module outputs a 1-D real representation with length $r$; (2) randomization module produces local differentially private representation; and (3) classifier module trains on the randomized binary representations to generate a differentially private classifier as per the post-processing invariance of DP~\cite{dwork2014algorithmic}. The detailed training process is summarised in Algorithm~\ref{Algorithm:dl_ldp}.

\begin{figure}[!htp]
\centering
        \includegraphics[width=8.1cm]{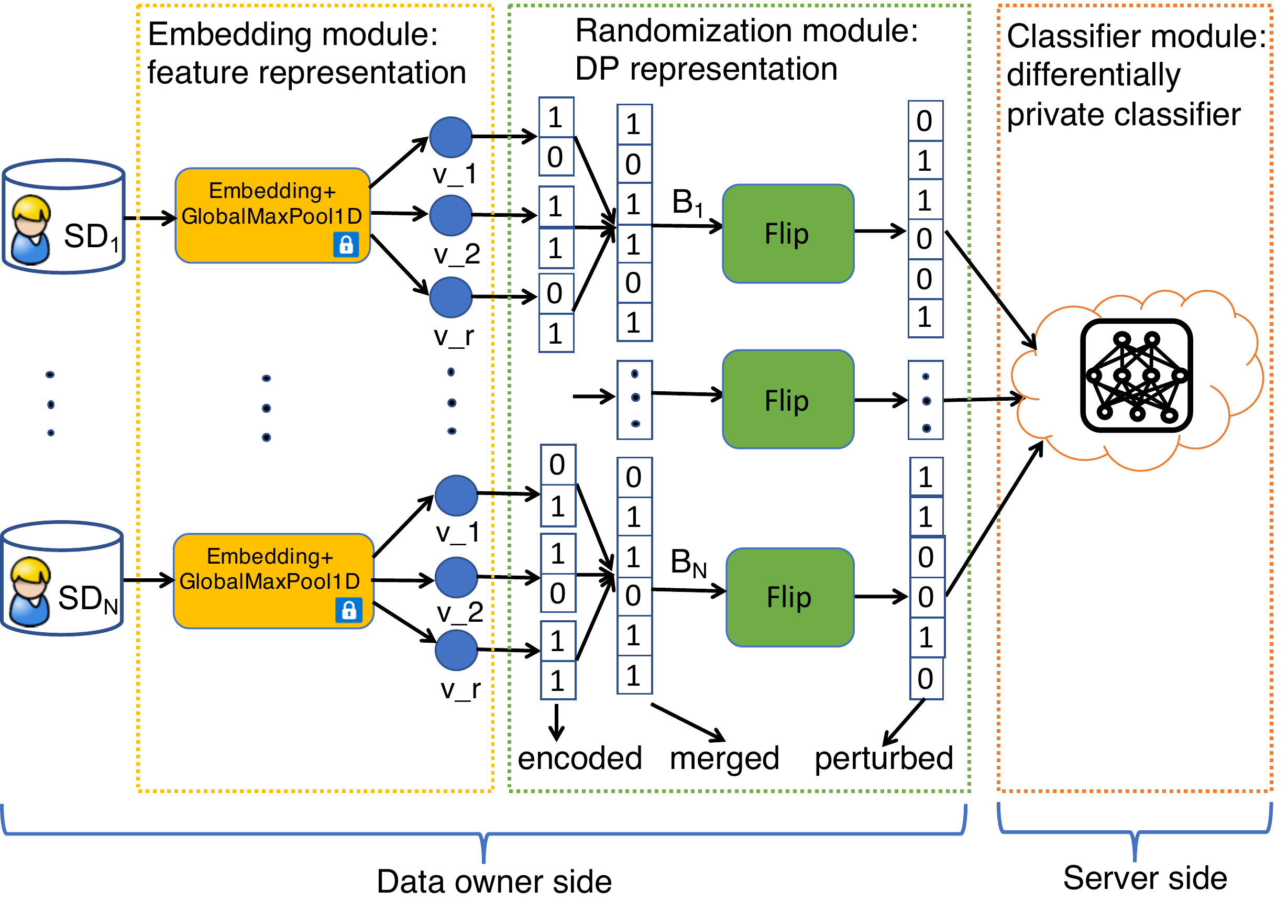}
        \caption{General setting for deep learning with LDP.
        }
\label{fig:LDPDL}
\end{figure}

\begin{algorithm}[htb]
\caption{Deep Learning with LDP}\label{Algorithm:dl_ldp}
\small
\begin{algorithmic}
  \State 1: Embedding: Each user maps its input into a 1-D embedding vector (with length $r$) using a pretrained embedding module (GloVe word embeddings, BERT, etc), z-score normalization is applied to avoid large values; 
  \State 2: Encoding: Each element $v_i \in \mathbb{R}^r$ of the normalized representation is mapped into a binary vector of size $l=1+m+n$. The first 1 bit represents the sign of the input (1 for negative and 0 for positive), the integer and fraction part are represented by the remaining $m$ and $n$ bits respectively;
  \State 3: Merging and perturbation: Each user merges all the $r$ binary vectors each with length $l$ into a long binary vector with total length $rl$, which is then randomized by using our proposed OME in Theorem~\ref{theorem:MOME}; 
  \State 4: Train target model: The server trains a differentially private classifier on all the received noisy representations.
\end{algorithmic}
\end{algorithm}

\section{Performance Evaluation}
\label{sec:Performance}
For performance evaluation, we focus on a range of NLP tasks: 1) sentiment analysis; 2) intent detection; and 3) paraphrase identification. 
In these tasks, the original sentences might carry some sensitive information such as name entities or monetary descriptions. These private information should be protected, meanwhile the performance for these tasks should not be heavily penalised.

\textbf{Datasets.} 
For sentiment analysis, we use three datasets: IMDb, Amazon, and Yelp, derived from~\cite{kotzias2015group}, where each review is labelled with a binary sentiment (positive vs. negative). For all sentiment datasets, we perform a train:test split into 8:2. 

Intent detection aims to classify each query into seven intents. We derive Intent dataset from~\cite{coucke2018snips}, which consists of 13,784 training examples and 700 test examples in total.

For paraphrase identification, we use Microsoft Research Paraphrase Corpus (MRPC) from \cite{dolan2005automatically}. This task decides whether given two sentences are semantically equivalent. 
Following \citet{wang2018glue}, we partition this data into train/test (3.7k/1.7k).

\textbf{Model and Training.} To extract the intermediate features, we use GloVe word embeddings with a dimension size of 50~\cite{pennington2014glove} for sentiment analysis, and use the pretrained BERT-base~\cite{devlin2018bert} for both Intent and MRPC.

For binary encoding of the extracted features, we use 10 bits (1 bit for the sign, 4 bits for the integer part, and 5 bits for the fraction part) to represent each element of the embedding vector.

The classifier module is a multi-layer perceptron with 128 hidden units for sentiment analysis and 768 units for Intent and MRPC. For all datasets, we train the models for 50 epochs using SGD optimizer with learning rate 0.01, decay $10^{-6}$, momentum 0.9, and a batch size of 32, and apply a dropout rate of 0.5 on the representation. For each dataset, we average the results over 20 runs.

\textbf{Experimental Results.} We first compare our local differentially private NN (LDPNN) with the non-private NN (NPNN), where the randomization module is removed. Table~\ref{tbl:accuracy1} shows that our LDPNN delivers comparable or even better results than the NPNN across various privacy budgets $\epsilon$ when the randomization factor $\lambda \geq 50$. We hypothesise that LDP acts as a regularization technique to avoid overfitting. We conjecture another important reason is that the \emph{enlarged feature space} through encoding produces more powerful representation than the conventional 1-D output of the embedding layer. As LDPNN performance is directly related to the randomization probabilities $p$ and $q$, the higher $p$ and the lower $q$, the lower the randomization of the binary vector, and the better performance will be expected. When the embedding size $r$ and encoding size $l$ are fixed, $p$ is determined by the randomization factor $\lambda$, and $q$ is determined by both the privacy budget $\epsilon$ and randomization factor $\lambda$, as indicated in Equation~\ref{eq:MOME_p} and Equation~\ref{eq:MOME_q}. Hence we next investigate how $\epsilon$ and $\lambda$ impact the model accuracy.

\begin{table}[ht]
\caption{Accuracy of NPNN and LDPNN using OME.} 
\label{tbl:accuracy1}
\centering
\scalebox{0.9}{
\begin{tabular}{ccccccc}

\multirow{2}{*}{Parameter} & \multirow{2}{*}{Value} & \multicolumn{5}{c}{Accuracy [\%]}\\
\cmidrule{3-7}
   &  & IMDb & Amazon & Yelp & Intent & MRPC \\
\toprule
\multicolumn{2}{c}{NPNN} 
&70.67 &67.50&66.00 &94.17 &68.38\\ 
\midrule
 \multirow{4}{*}{$\epsilon$ $(\lambda=100)$} 
& 0.5  &65.33 &68.00&64.54 &91.28 &66.91\\ 
& 1  &67.33 &69.50&66.73 &91.35 &67.15 \\ 
& 5  &67.80 &71.00&67.50 &91.57 &70.10\\ 
& 10 &69.33 &72.50&68.10 &91.87 &70.15 \\ 
\midrule
 \multirow{4}{*}{$\lambda$ $(\epsilon=1)$} 
&$1$ &48.00 &45.50&42.50 &13.00 &64.95\\ 
&$10$ &64.00 &65.46&57.81 &85.43 &66.66\\ 
&$50$ &66.67 &69.21&66.50 &90.57 &66.91 \\ 
&$100$ &67.33 &69.50&66.73 &91.35 &67.15 \\ 
\bottomrule
\end{tabular}}
\end{table}

\textbf{Impact of $\epsilon$}. Contrary to the heuristic study in deep learning with CDP~\cite{abadi2016deep}, from Table~\ref{tbl:accuracy1}, we observe that accuracies are relatively stable when the privacy budget $\epsilon$ is changed within a wide range of values. The reason lies in the large sensitivity of the encoded binary representation. When $\lambda$ is a constant, the large sensitivity $rl$ in OME weakens the effect of $\epsilon$ on the randomization probabilities, as evidenced by Figure~\ref{fig:probabilities_eps_lambda} (left), $p=\{p_1, p_2\}$ and $q$ of OME keep nearly consistent when $\epsilon$ changes. This also explains 
high accuracy even under a very tight privacy budget (e.g. $\epsilon$=0.5). 

We also compare with the other two LDP protocols (SUE and OUE in Section~\ref{sec:ldp_protocols}) on sentiment analysis task. Table~\ref{tbl:accuracy2} shows that our OME significantly outperforms both SUE and OUE. The reason lies in the optimized randomization probabilities of OME, as shown in Figure~\ref{fig:probabilities_eps_lambda} (left), $p$ and $q$ in both SUE and OUE are fluctuating around 0.5, causing low accuracies.

\begin{table}[ht]
\caption{Comparison with other LDP protocols.} 
\label{tbl:accuracy2}
\centering
\begin{tabular}{lccc}

\multirow{2}{*}{LDP protocols ($\epsilon=1$)} & \multicolumn{3}{c}{Accuracy [\%]}\\
\cmidrule{2-4}
  & IMDb & Amazon & Yelp \\
\toprule
SUE  &55.33 &45.50 &48.00  \\ 
OUE  &50.67 &54.00 &51.00  \\ 
OME($\lambda=100$, ours)  &\bfseries 67.33 &\bfseries 69.50&\bfseries 66.73  \\ 
\bottomrule
\end{tabular}
\end{table}

\textbf{Impact of $\lambda$}. For randomization factor $\lambda$, according to Equation~\ref{eq:MOME_q}, without the randomization factor $\lambda$, \ie $\lambda=1$, lower $\epsilon$ values and higher $rl$ values will result in higher $q$ values, which may compromise utility. To alleviate this problem, OME calibrates the value of $\lambda$ to adjust randomization probabilities. As observed in Figure~\ref{fig:probabilities_eps_lambda} (right), with the increasing $\lambda$, OME can largely decrease $q$ -- the probability of perturbing 0 to 1. Although the probability $p_2$ of preserving the original 1's decreases for half of the bit vector, the corresponding probability $p_1$ increases for the other half. This partially explains why OME can maintain both privacy and utility.

\begin{figure}[!htp]
\centering
\resizebox{\columnwidth}{!}{
\begin{tikzpicture}

\definecolor{color0}{rgb}{0.501960784313725,0.501960784313725,0}
\definecolor{color1}{rgb}{0,0.75,0.75}
\definecolor{color2}{rgb}{0.75,0,0.75}
\definecolor{color3}{rgb}{0.75,0.75,0}

\begin{groupplot}[group style={group size=2 by 1}]
\nextgroupplot[
legend cell align={left},
legend style={legend style={font=\LARGE}, fill opacity=0.8, draw opacity=1, text opacity=1, at={(2.1,0.7)}, anchor=north east, draw=white!20!black},
tick align=outside,
tick pos=left,
label style={font=\huge},
tick label style={font=\LARGE},
x grid style={white!69.0196078431373!black},
xlabel={Privacy budget \(\displaystyle \epsilon\)},
xmin=-0.15, xmax=3.15,
xtick style={color=black},
grid=both,
xtick={0,1,2,3},
xticklabels={0.5,1,5,10},
y grid style={white!69.0196078431373!black},
ylabel={Randomization probabilities},
ymin=-0.0495039004960995, ymax=1.03960391039609,
ytick style={color=black},
ytick={0,0.1,0.2,0.3,0.4,0.5,0.6,0.7,0.8,0.9,1},
yticklabels={0.0,0.1,0.2,0.3,0.4,0.5,0.6,0.7,0.8,0.9,1.0}
]
\addplot [thick, red, mark=*, mark size=3, mark options={solid}]
table {%
0 0.99009900990099
1 0.99009900990099
2 0.99009900990099
3 0.99009900990099
};
\addlegendentry{$p_1$}
\addplot [thick, color0, mark=asterisk, mark size=3, mark options={solid}]
table {%
0 9.99999000001e-07
1 9.99999000001e-07
2 9.99999000001e-07
3 9.99999000001e-07
};
\addlegendentry{$p_2$}
\addplot [thick, green!50!black, mark=x, mark size=3, mark options={solid}]
table {%
0 0.00989208228618355
1 0.00988318240762078
2 0.00981226818303343
3 0.00972433348803029
};
\addlegendentry{$q$}
\addplot [thick, red, mark=triangle*, mark size=3, mark options={solid,rotate=180}]
table {%
0 0.50124999739584
1 0.502499979166875
2 0.51249739648421
3 0.52497918747894
};
\addlegendentry{$p_{OUE}$}
\addplot [thick, color1, mark=triangle*, mark size=3, mark options={solid}]
table {%
0 0.49875000260416
1 0.497500020833125
2 0.48750260351579
3 0.47502081252106
};
\addlegendentry{$q_{OUE}$}
\addplot [thick, color2, mark=triangle*, mark size=3, mark options={solid,rotate=270}]
table {%
0 0.5
1 0.5
2 0.5
3 0.5
};
\addlegendentry{$p_{SUE}$}
\addplot [thick, color3, mark=triangle*, mark size=3, mark options={solid,rotate=90}]
table {%
0 0.49875000260416
1 0.497500020833125
2 0.48750260351579
3 0.47502081252106
};
\addlegendentry{$q_{SUE}$}

\nextgroupplot[
legend cell align={left},
legend style={fill opacity=0.8, draw opacity=1, text opacity=1, at={(0.91,0.5)}, anchor=east, draw=white!80!black},
tick align=outside,
tick pos=left,
x grid style={white!69.0196078431373!black},
xlabel={Randomization factor \(\displaystyle \lambda\)},
xmin=-0.15, xmax=3.15,
xtick style={color=black},
xtick={0,1,2,3},
xticklabels={1,10,50,100},
y grid style={white!69.0196078431373!black},
ymin=-0.0495039004960995, ymax=1.03960391039609,
ytick style={color=black},
ytick={0,0.1,0.2,0.3,0.4,0.5,0.6,0.7,0.8,0.9,1},
yticklabels={,},
grid=both,
label style={font=\huge},
tick label style={font=\LARGE},
]
\addplot [thick, red, mark=*, mark size=3, mark options={solid}]
table {%
0 0.5
1 0.909090909090909
2 0.980392156862745
3 0.99009900990099
};
\addlegendentry{$p_1$}
\addplot [thick, color0, mark=asterisk, mark size=3, mark options={solid}]
table {%
0 0.5
1 0.000999000999000999
2 7.999936000512e-06
3 9.99999000001e-07
};
\addlegendentry{$p_2$}
\addplot [thick, green!50!black, mark=x, mark size=3, mark options={solid}]
table {%
0 0.499545454670674
1 0.0907589396730121
2 0.0195729220563089
3 0.00988318240762078
};
\addlegendentry{$q$}
\legend{};
\end{groupplot}

\end{tikzpicture}}
        \caption{Randomization probabilities change with: (left) $\epsilon$, (right) $\lambda$, where $p_1=\Pr\{1 \rightarrow 1\} \ $for$\ i \in 2n$, $p_2=\Pr\{1 \rightarrow 1\} \ $for$\ i \in 2n+1$, $q=\Pr\{0 \rightarrow 1\}$ in OME.}
\label{fig:probabilities_eps_lambda}
\end{figure}
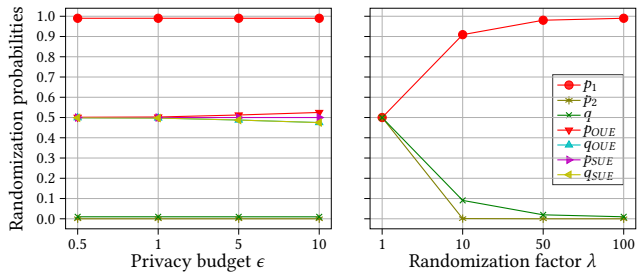

Overall, all these results show that our OME offers the reduced impact of the privacy budget $\epsilon$ on model accuracy, and significantly outperforms the most state-of-the-art LDP protocols. 

\section{CONCLUSION}
We formulated a new deep learning framework, which allows data owners to send differentially private representations for further learning on the untrusted servers. A novel LDP protocol was proposed to adjust the randomization probabilities of the binary representation while maintaining both high privacy and accuracy under various privacy budgets. Experimental results on a range of NLP tasks confirm the effectiveness and superiority of our framework.

\bibliographystyle{ACM-Reference-Format}
\bibliography{ldp_biblio}


\begin{thebibliography}{20}


\ifx \showCODEN    \undefined \def \showCODEN     #1{\unskip}     \fi
\ifx \showDOI      \undefined \def \showDOI       #1{#1}\fi
\ifx \showISBNx    \undefined \def \showISBNx     #1{\unskip}     \fi
\ifx \showISBNxiii \undefined \def \showISBNxiii  #1{\unskip}     \fi
\ifx \showISSN     \undefined \def \showISSN      #1{\unskip}     \fi
\ifx \showLCCN     \undefined \def \showLCCN      #1{\unskip}     \fi
\ifx \shownote     \undefined \def \shownote      #1{#1}          \fi
\ifx \showarticletitle \undefined \def \showarticletitle #1{#1}   \fi
\ifx \showURL      \undefined \def \showURL       {\relax}        \fi
\providecommand\bibfield[2]{#2}
\providecommand\bibinfo[2]{#2}
\providecommand\natexlab[1]{#1}
\providecommand\showeprint[2][]{arXiv:#2}

\bibitem[\protect\citeauthoryear{Abadi, Chu, Goodfellow, McMahan, Mironov,
  Talwar, and Zhang}{Abadi et~al\mbox{.}}{2016}]%
        {abadi2016deep}
\bibfield{author}{\bibinfo{person}{Mart{\'\i}n Abadi}, \bibinfo{person}{Andy
  Chu}, \bibinfo{person}{Ian Goodfellow}, \bibinfo{person}{H~Brendan McMahan},
  \bibinfo{person}{Ilya Mironov}, \bibinfo{person}{Kunal Talwar}, {and}
  \bibinfo{person}{Li Zhang}.} \bibinfo{year}{2016}\natexlab{}.
\newblock \showarticletitle{Deep learning with differential privacy}. In
  \bibinfo{booktitle}{\emph{Proceedings of CCS}}. ACM,
  \bibinfo{pages}{308--318}.
\newblock


\bibitem[\protect\citeauthoryear{Bonawitz, Ivanov, Kreuter, Marcedone, McMahan,
  Patel, Ramage, Segal, and Seth}{Bonawitz et~al\mbox{.}}{2017}]%
        {bonawitz2017practical}
\bibfield{author}{\bibinfo{person}{Keith Bonawitz}, \bibinfo{person}{Vladimir
  Ivanov}, \bibinfo{person}{Ben Kreuter}, \bibinfo{person}{Antonio Marcedone},
  \bibinfo{person}{H~Brendan McMahan}, \bibinfo{person}{Sarvar Patel},
  \bibinfo{person}{Daniel Ramage}, \bibinfo{person}{Aaron Segal}, {and}
  \bibinfo{person}{Karn Seth}.} \bibinfo{year}{2017}\natexlab{}.
\newblock \showarticletitle{Practical Secure Aggregation for Privacy-Preserving
  Machine Learning}. In \bibinfo{booktitle}{\emph{Proceedings of CCS}}. ACM,
  \bibinfo{pages}{1175--1191}.
\newblock


\bibitem[\protect\citeauthoryear{Coavoux, Narayan, and Cohen}{Coavoux
  et~al\mbox{.}}{2018}]%
        {coavoux2018privacy}
\bibfield{author}{\bibinfo{person}{Maximin Coavoux}, \bibinfo{person}{Shashi
  Narayan}, {and} \bibinfo{person}{Shay~B Cohen}.}
  \bibinfo{year}{2018}\natexlab{}.
\newblock \showarticletitle{Privacy-preserving neural representations of text}.
  In \bibinfo{booktitle}{\emph{Proceedings of EMNLP}}. \bibinfo{pages}{1--10}.
\newblock


\bibitem[\protect\citeauthoryear{Coucke, Saade, Ball, Bluche, Caulier, Leroy,
  Doumouro, Gisselbrecht, Caltagirone, Lavril, et~al\mbox{.}}{Coucke
  et~al\mbox{.}}{2018}]%
        {coucke2018snips}
\bibfield{author}{\bibinfo{person}{Alice Coucke}, \bibinfo{person}{Alaa Saade},
  \bibinfo{person}{Adrien Ball}, \bibinfo{person}{Th{\'e}odore Bluche},
  \bibinfo{person}{Alexandre Caulier}, \bibinfo{person}{David Leroy},
  \bibinfo{person}{Cl{\'e}ment Doumouro}, \bibinfo{person}{Thibault
  Gisselbrecht}, \bibinfo{person}{Francesco Caltagirone},
  \bibinfo{person}{Thibaut Lavril}, {et~al\mbox{.}}}
  \bibinfo{year}{2018}\natexlab{}.
\newblock \showarticletitle{Snips voice platform: an embedded spoken language
  understanding system for private-by-design voice interfaces}.
\newblock \bibinfo{journal}{\emph{arXiv preprint arXiv:1805.10190}}
  (\bibinfo{year}{2018}).
\newblock


\bibitem[\protect\citeauthoryear{Devlin, Chang, Lee, and Toutanova}{Devlin
  et~al\mbox{.}}{2018}]%
        {devlin2018bert}
\bibfield{author}{\bibinfo{person}{Jacob Devlin}, \bibinfo{person}{Ming-Wei
  Chang}, \bibinfo{person}{Kenton Lee}, {and} \bibinfo{person}{Kristina
  Toutanova}.} \bibinfo{year}{2018}\natexlab{}.
\newblock \showarticletitle{Bert: Pre-training of deep bidirectional
  transformers for language understanding}.
\newblock \bibinfo{journal}{\emph{arXiv preprint arXiv:1810.04805}}
  (\bibinfo{year}{2018}).
\newblock


\bibitem[\protect\citeauthoryear{Dolan and Brockett}{Dolan and
  Brockett}{2005}]%
        {dolan2005automatically}
\bibfield{author}{\bibinfo{person}{William~B Dolan} {and}
  \bibinfo{person}{Chris Brockett}.} \bibinfo{year}{2005}\natexlab{}.
\newblock \showarticletitle{Automatically constructing a corpus of sentential
  paraphrases}. In \bibinfo{booktitle}{\emph{Proceedings of the Third
  International Workshop on Paraphrasing (IWP2005)}}.
\newblock


\bibitem[\protect\citeauthoryear{Duchi, Jordan, and Wainwright}{Duchi
  et~al\mbox{.}}{2013}]%
        {duchi2013local}
\bibfield{author}{\bibinfo{person}{John~C Duchi}, \bibinfo{person}{Michael~I
  Jordan}, {and} \bibinfo{person}{Martin~J Wainwright}.}
  \bibinfo{year}{2013}\natexlab{}.
\newblock \showarticletitle{Local privacy and statistical minimax rates}. In
  \bibinfo{booktitle}{\emph{2013 IEEE 54th Annual Symposium on Foundations of
  Computer Science}}. IEEE, \bibinfo{pages}{429--438}.
\newblock


\bibitem[\protect\citeauthoryear{Dwork and Roth}{Dwork and Roth}{2014}]%
        {dwork2014algorithmic}
\bibfield{author}{\bibinfo{person}{Cynthia Dwork} {and} \bibinfo{person}{Aaron
  Roth}.} \bibinfo{year}{2014}\natexlab{}.
\newblock \showarticletitle{The algorithmic foundations of differential
  privacy}.
\newblock \bibinfo{journal}{\emph{Foundations and Trends{\textregistered} in
  Theoretical Computer Science}} \bibinfo{volume}{9}, \bibinfo{number}{3--4}
  (\bibinfo{year}{2014}), \bibinfo{pages}{211--407}.
\newblock


\bibitem[\protect\citeauthoryear{Hovy, Johannsen, and S{\o}gaard}{Hovy
  et~al\mbox{.}}{2015}]%
        {hovy2015user}
\bibfield{author}{\bibinfo{person}{Dirk Hovy}, \bibinfo{person}{Anders
  Johannsen}, {and} \bibinfo{person}{Anders S{\o}gaard}.}
  \bibinfo{year}{2015}\natexlab{}.
\newblock \showarticletitle{User review sites as a resource for large-scale
  sociolinguistic studies}. In \bibinfo{booktitle}{\emph{Proceedings of WWW}}.
  \bibinfo{pages}{452--461}.
\newblock


\bibitem[\protect\citeauthoryear{Kotzias, Denil, De~Freitas, and Smyth}{Kotzias
  et~al\mbox{.}}{2015}]%
        {kotzias2015group}
\bibfield{author}{\bibinfo{person}{Dimitrios Kotzias}, \bibinfo{person}{Misha
  Denil}, \bibinfo{person}{Nando De~Freitas}, {and} \bibinfo{person}{Padhraic
  Smyth}.} \bibinfo{year}{2015}\natexlab{}.
\newblock \showarticletitle{From group to individual labels using deep
  features}. In \bibinfo{booktitle}{\emph{Proceedings of the SIGKDD}}. ACM,
  \bibinfo{pages}{597--606}.
\newblock


\bibitem[\protect\citeauthoryear{Li, Baldwin, and Cohn}{Li
  et~al\mbox{.}}{2018}]%
        {li2018towards}
\bibfield{author}{\bibinfo{person}{Yitong Li}, \bibinfo{person}{Timothy
  Baldwin}, {and} \bibinfo{person}{Trevor Cohn}.}
  \bibinfo{year}{2018}\natexlab{}.
\newblock \showarticletitle{Towards robust and privacy-preserving text
  representations}. In \bibinfo{booktitle}{\emph{Proceedings of ACL}}.
  \bibinfo{pages}{25--30}.
\newblock


\bibitem[\protect\citeauthoryear{Lyu, Yu, and Yang}{Lyu et~al\mbox{.}}{2020a}]%
        {lyu2020threats}
\bibfield{author}{\bibinfo{person}{Lingjuan Lyu}, \bibinfo{person}{Han Yu},
  {and} \bibinfo{person}{Qiang Yang}.} \bibinfo{year}{2020}\natexlab{a}.
\newblock \showarticletitle{Threats to Federated Learning: A Survey}.
\newblock \bibinfo{journal}{\emph{arXiv preprint arXiv:2003.02133}}
  (\bibinfo{year}{2020}).
\newblock


\bibitem[\protect\citeauthoryear{Lyu, Yu, Nandakumar, Li, Ma, Jin, Yu, and
  Ng}{Lyu et~al\mbox{.}}{2020b}]%
        {lyu2020towards}
\bibfield{author}{\bibinfo{person}{Lingjuan Lyu}, \bibinfo{person}{Jiangshan
  Yu}, \bibinfo{person}{Karthik Nandakumar}, \bibinfo{person}{Yitong Li},
  \bibinfo{person}{Xingjun Ma}, \bibinfo{person}{Jiong Jin},
  \bibinfo{person}{Han Yu}, {and} \bibinfo{person}{Kee~Siong Ng}.}
  \bibinfo{year}{2020}\natexlab{b}.
\newblock \showarticletitle{Towards Fair and Privacy-Preserving Federated Deep
  Models}.
\newblock \bibinfo{journal}{\emph{IEEE TPDS}} \bibinfo{volume}{31},
  \bibinfo{number}{11} (\bibinfo{year}{2020}), \bibinfo{pages}{2524--2541}.
\newblock


\bibitem[\protect\citeauthoryear{McMahan, Moore, Ramage, Hampson,
  et~al\mbox{.}}{McMahan et~al\mbox{.}}{2017}]%
        {mcmahan2016communication}
\bibfield{author}{\bibinfo{person}{H~Brendan McMahan}, \bibinfo{person}{Eider
  Moore}, \bibinfo{person}{Daniel Ramage}, \bibinfo{person}{Seth Hampson},
  {et~al\mbox{.}}} \bibinfo{year}{2017}\natexlab{}.
\newblock \showarticletitle{Communication-efficient learning of deep networks
  from decentralized data}.
\newblock \bibinfo{journal}{\emph{AISTATS}} (\bibinfo{year}{2017}).
\newblock


\bibitem[\protect\citeauthoryear{Pennington, Socher, and Manning}{Pennington
  et~al\mbox{.}}{2014}]%
        {pennington2014glove}
\bibfield{author}{\bibinfo{person}{Jeffrey Pennington},
  \bibinfo{person}{Richard Socher}, {and} \bibinfo{person}{Christopher
  Manning}.} \bibinfo{year}{2014}\natexlab{}.
\newblock \showarticletitle{Glove: Global vectors for word representation}. In
  \bibinfo{booktitle}{\emph{Proceedings of EMNLP}}.
  \bibinfo{pages}{1532--1543}.
\newblock


\bibitem[\protect\citeauthoryear{Preo{\c{t}}iuc-Pietro, Lampos, and
  Aletras}{Preo{\c{t}}iuc-Pietro et~al\mbox{.}}{2015}]%
        {preoctiuc2015analysis}
\bibfield{author}{\bibinfo{person}{Daniel Preo{\c{t}}iuc-Pietro},
  \bibinfo{person}{Vasileios Lampos}, {and} \bibinfo{person}{Nikolaos
  Aletras}.} \bibinfo{year}{2015}\natexlab{}.
\newblock \showarticletitle{An analysis of the user occupational class through
  Twitter content}. In \bibinfo{booktitle}{\emph{Proceedings of ACL}},
  Vol.~\bibinfo{volume}{1}. \bibinfo{pages}{1754--1764}.
\newblock


\bibitem[\protect\citeauthoryear{Shokri and Shmatikov}{Shokri and
  Shmatikov}{2015}]%
        {shokri2015privacy}
\bibfield{author}{\bibinfo{person}{Reza Shokri} {and} \bibinfo{person}{Vitaly
  Shmatikov}.} \bibinfo{year}{2015}\natexlab{}.
\newblock \showarticletitle{Privacy-preserving deep learning}. In
  \bibinfo{booktitle}{\emph{Proceedings of CCS}}. ACM,
  \bibinfo{pages}{1310--1321}.
\newblock


\bibitem[\protect\citeauthoryear{Wang, Singh, Michael, Hill, Levy, and
  Bowman}{Wang et~al\mbox{.}}{018b}]%
        {wang2018glue}
\bibfield{author}{\bibinfo{person}{Alex Wang}, \bibinfo{person}{Amanpreet
  Singh}, \bibinfo{person}{Julian Michael}, \bibinfo{person}{Felix Hill},
  \bibinfo{person}{Omer Levy}, {and} \bibinfo{person}{Samuel~R Bowman}.}
  \bibinfo{year}{2018b}\natexlab{}.
\newblock \showarticletitle{Glue: A multi-task benchmark and analysis platform
  for natural language understanding}.
\newblock \bibinfo{journal}{\emph{arXiv preprint arXiv:1804.07461}}
  (\bibinfo{year}{2018b}).
\newblock


\bibitem[\protect\citeauthoryear{Wang, Blocki, Li, and Jha}{Wang
  et~al\mbox{.}}{2017}]%
        {wang2017locally}
\bibfield{author}{\bibinfo{person}{Tianhao Wang}, \bibinfo{person}{Jeremiah
  Blocki}, \bibinfo{person}{Ninghui Li}, {and} \bibinfo{person}{Somesh Jha}.}
  \bibinfo{year}{2017}\natexlab{}.
\newblock \showarticletitle{Locally differentially private protocols for
  frequency estimation}. In \bibinfo{booktitle}{\emph{USENIX Security}}.
  \bibinfo{pages}{729--745}.
\newblock


\bibitem[\protect\citeauthoryear{Warner}{Warner}{1965}]%
        {warner1965randomized}
\bibfield{author}{\bibinfo{person}{Stanley~L Warner}.}
  \bibinfo{year}{1965}\natexlab{}.
\newblock \showarticletitle{Randomized response: A survey technique for
  eliminating evasive answer bias}.
\newblock \bibinfo{journal}{\emph{J. Amer. Statist. Assoc.}}
  \bibinfo{volume}{60}, \bibinfo{number}{309} (\bibinfo{year}{1965}),
  \bibinfo{pages}{63--69}.
\newblock


\end{thebibliography}
\end{document}